\documentclass[11pt,twoside]{article}
\usepackage{hyperref}
\usepackage{xspace}
\usepackage{amssymb}
\usepackage{amsmath}

\newcommand{\ie}{i.e.\ }
\newcommand{\eg}{e.g.\ }

\newcommand{\A}{{\mathcal A}}

\newcommand{\U}{\mathcal{U}}

\newcommand{\R}{{\mathbb R}}
\newcommand{\E}{{\mathbf E}}

\newcommand{\pr}{\mathbf{Pr}}
\newcommand{\ra}{\rangle}
\newcommand{\la}{\langle}
\newcommand{\cond}{\ |\ }
\newcommand{\zo}{\{0,1\}}
\newcommand{\zon}{\{0,1\}^n}

\newcommand{\pmi}{\{-1,1\}}
\newcommand{\pmn}{\{-1,1\}^n}
\newcommand{\pmr}{[-1,1]}

\newcommand{\poly}{\mbox{poly}}

\newcommand{\ti} \tilde

\newcommand{\sgn}{\mathsf{sign}}
\newcommand{\eps}{\epsilon}

\newcommand{\exoracle}[2]{\mbox{EX}(#2,#1)}

\newcommand{\logd}{\log{(1/\delta)}}

\newcommand{\equ}[1]{

\begin{equation}
#1
\end{equation}}
\newcommand{\equn}[1]{
$$
#1
$$}

\newcommand{\ignore}[1]{\relax}
\newcommand{\alequ}[1]{\begin{align} #1 \end{align}}
\newcommand{\alequn}[1]{\begin{align*} #1 \end{align*}}

\newcommand{\eat}[1]{}

\newtheorem{theorem}{Theorem}[section]
\newtheorem{lemma}[theorem]{Lemma}

\newtheorem{corollary}[theorem]{Corollary}
\newtheorem{remark}[theorem]{Remark}

\newtheorem{definition}[theorem]{Definition}

\usepackage[round,comma]{natbib}
\usepackage{fullpage}

\newcommand{\ptfreconprod}{{\tt PTFconstructProd}\xspace}
\newcommand{\ptfapprox}{{\tt PTFapprox}\xspace}
\newcommand{\ptfapproxprod}{{\tt PTFapproxProd}\xspace}
\newcommand{\dnflearnprod}{{\tt DNFLearnMQProd}\xspace}
\newcommand{\mdnflearnprod}{{\tt MDNFLearnProd}\xspace}
\newenvironment{proof}{\noindent \textbf{Proof:}}{\hfill{$\Box$}}

\title{Learning DNF Expressions from Fourier Spectrum}

\author{Vitaly Feldman \\
IBM Almaden Research Center\\
{\tt vitaly@post.harvard.edu}}

\begin{document}

\maketitle

\begin{abstract}
Since its introduction by Valiant in 1984, PAC learning of DNF expressions remains one of the central problems in learning theory. We consider this problem in the setting where the underlying distribution is uniform, or more generally, a product distribution. \citet*{KalaiST:09} showed that in this setting a DNF expression can be efficiently approximated from its ``heavy" low-degree Fourier coefficients alone. This is in contrast to previous approaches where boosting was used and thus Fourier coefficients of the target function modified by various distributions were needed. This property is crucial for learning of DNF expressions over smoothed product distributions, a learning model introduced by \citet{KalaiST:09} and inspired by the seminal smoothed analysis model of \citet{SpielmanTeng:04}.

We introduce a new approach to learning (or approximating) a polynomial threshold functions which is based on creating a function with range $[-1,1]$ that approximately agrees with the unknown function on low-degree Fourier coefficients. We then describe conditions under which this is sufficient for learning polynomial threshold functions. As an application of our approach, we give a new, simple algorithm for approximating any polynomial-size DNF expression from its ``heavy" low-degree Fourier coefficients alone. Our algorithm greatly simplifies the proof of learnability of DNF expressions over smoothed product distributions and is simpler than all previous algorithm for PAC learning of DNF expression using membership queries. We also describe an application of our algorithm to learning monotone DNF expressions over product distributions. Building on the work of \citet{Servedio:04mondnf}, we give an algorithm that runs in time $\poly((s \cdot \log{(s/\eps)})^{\log{(s/\eps)}}, n)$, where $s$ is the size of the DNF expression and $\eps$ is the accuracy. This improves on $\poly((s \cdot \log{(ns/\eps)})^{\log{(s/\eps)} \cdot \log{(1/\eps)}}, n)$ bound of \citet{Servedio:04mondnf}. Another advantage of our algorithm is that it can be applied to a large class of polynomial threshold functions whereas previous algorithms for both applications relied on the function being a polynomial-size DNF expression.
\end{abstract}

\eat{
}

\section{Introduction} \label{sec:intro}
PAC learning of DNF expressions (or formulae) is the problem posed by \citet{Valiant:84} in his seminal work that introduced the PAC model. The original problem asks whether polynomial-size DNF expressions are learnable from random examples on points sampled from an unknown distribution. Despite efforts by numerous researchers, the problem still remains open, with the best algorithm taking $2^{\tilde{O}(\sqrt[3]{n})}$ time \citep{KlivansServedio:04}. In the course of this work, a number of restricted versions of the problem were introduced and studied. One such assumption is that the distribution over the domain (which is the $n$-dimensional hypercube $\pmi^n$) is uniform, or more generally, a product distribution. In this setting a simple quasi-polynomial $n^{O(\log n)}$ algorithm for learning DNF expressions was found by \citet{Verbeurgt:90}. However, no substantially better algorithms are known so far even for much simpler classes such as functions of at most $\log{n}$-variables ($\log{n}$-juntas).

Another natural restriction commonly considered is monotone DNF (MDNF) expressions, \ie those without negated variables. Without restrictions on the distribution, the problem is no easier than the original one \citep{KearnsLPV:87b} but appears to be easier for product distributions. \citet{SakaiMaruoka:00} gave a polynomial-time algorithm for $\log{n}$-term MDNF learning and \citet{BshoutyTamon:96} gave an algorithm for learning a class of functions which includes $O(\log^2{n}/\log\log{n})$-term MDNFs. Most recently, \citet{Servedio:04mondnf} proved a substantially stronger result: $s$-term MDNFs are learnable to accuracy $\eps$ in time polynomial in $(s \cdot \log{(ns/\eps)})^{\log{(s/\eps)}\cdot \log{(1/\eps)}}$ and $n$. In particular, his result implies that $O(2^{\sqrt{\log{n}}})$-term MDNFs are learnable in polynomial time to any constant accuracy. Numerous other restrictions of the original problem were considered. We refer the interested reader to Servedio's paper \citeyearpar{Servedio:04mondnf} for a more detailed overview.

Several works also considered the problem in the stronger {\em membership query} (MQ) model. In this model the learner can ask for a value of the unknown function at any point in the domain. \citet{Valiant:84} gave an efficient MQ learning algorithm for MDNFs of polynomial size. In a celebrated result, \citet{Jackson:97} gave a polynomial time MQ learning algorithm for DNFs over product distributions. Jackson's algorithm uses the Fourier transform-based learning technique \citep{LinialMN:93} and combines the Kushilevitz-Mansour algorithm for finding a ``heavy" Fourier coefficient of a boolean function \citep{GoldreichLevin:89,KushilevitzMansour:93} with the Boosting-by-Majority algorithm of \citet{Freund:95}. A similar approach was used in the subsequent improvements to Jackson's algorithm \citep{KlivansServedio:03,BJT:04,Feldman:07jmlr}.

The access to membership queries is clearly a very strong assumption and is unrealistic in most learning applications. Several works give DNF learning algorithms which relax this requirement: the learning algorithm of \citet{BshoutyFeldman:02} uses random examples from product distributions chosen by the algorithm and the algorithm of \citet{BshoutyMOS:05} uses only examples produced by a random walk on the hypercube. Another approach is to relax the requirement that the PAC algorithm succeeds on all polynomial-size DNF formulae and require it to succeed on a randomly chosen expression generated from some simple distribution over the formulae \citep{AizensteinPitt:95}. Strong results of this form were achieved recently by \citet{JacksonLSW:11} and \citet{Sellie:09}.

A new way to avoid the worst-case hardness of learning DNF was recently proposed by \citet{KalaiST:09}. Their model is inspired by the seminal model of smoothed analysis introduced in the context of optimization and numerical analysis by \citet{SpielmanTeng:04}. Smoothed analysis is based on the insight that, in practice, real-valued inputs or parameters of the problem are a result of noisy and imprecise measurements. Therefore the complexity of a problem is measured not on the worst-case values but on a random perturbation of those values. In the work of \citet{KalaiST:09} the perturbed parameters are the expectations of each of the coordinates of a product distribution over $\pmi^n$. In a surprising result they showed that DNF formulae are learnable efficiently in this model (and that decision trees are even learnable agnostically).

A crucial and the most involved component of the DNF learning algorithm of \citet{KalaiST:09} is the algorithm that -- given all ``heavy" (here this refers to those of inverse-polynomial magnitude), low-degree (logarithmic in the learning parameters) Fourier coefficients of the target DNF $f$ to inverse-polynomial accuracy -- finds a function that is $\eps$-close to $f$. Such an algorithm is necessary since, in the boosting-based approach of \citet{Jackson:97}, the weak learner needs to learn with respect to distributions which depend on previous weak hypotheses. When learning over a smoothed product distribution, the first weak hypothesis depends on the specific perturbation and therefore in the subsequent boosting stages, the parameters of the product distribution can no longer be thought of as perturbed randomly. \citet{KalaiST:09} show that this is not only a matter of complications in the analysis but an actual limitation of the boosting-based approach. Therefore they used an algorithm that first collects all the ``heavy" low-degree Fourier coefficients and then relies solely on this information to approximate the target function.

\subsection{Our Results}
We describe a new approach to the problem of learning a polynomial threshold function (PTF) from approximations of its ``heavy" low-degree Fourier coefficients, a problem we believe is interesting in its own right. The approach exploits a generalization of a simple structural result about any $s$-term DNF $f$: for every function $g:\pmn \rightarrow \pmr$, the error of $g$ on $f$ (measured as $\E_\U[|f(x)-g(x)|]$) is at most $\gamma \cdot (2s+1)$, where $\gamma$ is the magnitude of the largest difference between two corresponding Fourier coefficients of $f$ and $g$ \citep{KalaiST:09}. We use $\hat{f}$ to denote the vector of Fourier coefficients of $f$ and so this difference can be expressed as $\|\hat{f} - \hat{g}\|_\infty$. Hence to find a function $\eps$-close to $f$ it is sufficient to find a function $g$ such that $\|\hat{f} - \hat{g}\|_\infty \leq \eps/(2s+1)$, in other words, $g$ that has approximately (in the infinity norm) the same Fourier spectrum as $f$. We give a new, simple algorithm (Th.~\ref{th:ptf-reconstruct}) that constructs a function (with range in $[-1,1]$) which has approximately the desired Fourier spectrum.

Our algorithm builds $g$ in a fairly straightforward way: starting with a constant $g_0\equiv 0$ function we iteratively correct each coefficient to the desired value (by adding the difference in the coefficients multiplied by the corresponding basis function). After each such step the new function $g_t$ might have values outside of $\pmr$. We correct this by ``cutting-off" values outside of $\pmr$ (in other words, project them to $\pmr$). A simple argument shows that both of these operations reduce $\|f-g_t\|_2^2 = \E_\U[f(x)-g_t(x))^2]$. The coefficient correction procedure reduces this squared distance measure significantly and implies the convergence of the algorithm. In addition, through a slightly more complicated potential argument we show that there is  no need to perform the projection after each coefficient update; a single projection after all updates suffices (Th.~\ref{th:ptf-reconstruct-proper}). This implies that the function we construct via this algorithm is itself a polynomial threshold function (PTF).

To generalize our approach to product distributions, we strengthen the structural lemma about DNF expressions to measure the error in terms of the largest difference between corresponding low-degree Fourier coefficients and extend it to product distributions (Th.~\ref{th:dnf-fourier-approx-bound}). The algorithm itself uses the Fourier basis for the given product distribution but otherwise remains essentially unchanged. We also give a more general condition on PTFs that is sufficient for bounding $\E_\U[|f(x)-g(x)|]$ in terms of  largest difference between corresponding low-degree Fourier coefficients of $f$ and $g$. The general condition implies that our algorithm can also be used to learn any integer-weight linear threshold of terms as long as the sum of the magnitudes of weights (or the {\em total weight}) is polynomial.

We give several applications of our approach. The most immediate one is to obtain a simple algorithm for learning DNF expressions over product distributions with membership queries (Cor.~\ref{cor:learn-dnf-prod}). Given access to membership queries, the Fourier spectrum of any function can be approximated using the well-known Kushilevitz-Mansour algorithm and its generalization to product distributions \citep{GoldreichLevin:89,KushilevitzMansour:93}. We can then apply our approximation algorithm to get a hypothesis which is $\eps$ close to the target function. While technically our iterative algorithm is similar to boosting, the resulting algorithm for learning DNF is simpler and more self-contained than previous boosting-based algorithms.

The second application of our approximation algorithm and the motivation for this work is its use in the context of smoothed analysis of learning DNF over product distributions (Th.~\ref{th:smooth-learn-dnf}) where the problem was originally formulated and solved by \citet{KalaiST:09}. The approximation algorithm of \citet{KalaiST:09} is based on an elaborate combination of the {\em positive-reliable} DNF learning algorithm of \citet{KalaiKM:09} and the agnostic learning algorithm for decisions trees of \citet{GopalanKK:08}. In contrast, our algorithm gives a natural solution to the problem which is significantly simpler technically and is more general. We also note that the algorithm of \citet{KalaiST:09} does not construct a function with Fourier transform close to that of $f$ and is not based on the structural results we use.

In another application of our approach we give a new algorithm for learning MDNF expressions over product distributions. Our algorithm is based on Servedio's algorithm for learning MDNFs \citep{Servedio:04mondnf}. The main idea of his algorithm is to restrict the target function to influential variables, those that can change the value of the target function with significant probability. For any monotone function, influential variables can be easily identified. Then all the Fourier coefficients of low degree and restricted to influential variables are estimated individually from random examples. The sign of the resulting low-degree polynomial is used as a hypothesis. The degree for which such an approximation method is known to work is $20 \cdot \log{(s/\eps)} \cdot \log{(1/\eps)})$  \citep{Mansour:95}. Using our simple structural result about DNF and our algorithm for constructing a function with desired Fourier coefficients, we show  (Th.~\ref{th:learn-mdnf}) that to achieve $\eps$-accuracy coefficients of degree at most $O(\log{(s/\eps)})$ are sufficient. This results in $\poly((s \cdot \log{(s/\eps)})^{\log{(s/\eps)}}, n)$ time algorithm improving on $\poly((s \cdot \log{(ns/\eps)})^{\log{(s/\eps)} \cdot \log{(1/\eps)}}, n)$ bound of \citet{Servedio:04mondnf}.


\noindent {\bf Related work.} A closely related problem of finding a function with specified correlations with a given set of functions was considered by \citet{TrevisanTV:09} and their solution is based on a similar algorithm (with a more involved analysis). Our setting differs in that the set of functions with which correlations are specified has a superpolynomial size and the functions are not necessarily boolean (when the distribution is non-uniform).

In the Chow Parameter problem the goal is to find an approximation to a linear threshold function (LTF) $f$ from its degree-$1$ and degree-$0$ Fourier coefficients (the Chow parameters). \citet{ODonnellServedio:11} gave the first algorithm for the problem which is based on finding a function whose Chow parameters are close in Euclidean distance to those of $f$ (as opposed to $\|\cdot\|_\infty$ distance in our problem). Then they used an intricate structural result about LTFs to derive an approximation bound. Their algorithm is based on a brute-force search of some of the Chow parameters. A very recent, doubly exponential improvement to the solution of the problem was obtained using a new, stronger structural result and a new algorithm for constructing a linear threshold function from approximations of Chow parameters \citep{DeDFS:12}. As in our applications, the algorithm of \citet{DeDFS:12} constructs a bounded function with the given degree-1 Fourier spectrum. However the update step of their algorithm is optimized for minimizing the Euclidean distance of the Chow parameters of the obtained function to the given ones.

\noindent {\bf Organization.}
Structural results required for approximating DNF expressions and PTFs are given in Section \ref{sec:structural}. In Section \ref{sec:construct} we describe our main algorithm for constructing a function with the desired Fourier spectrum. In Section \ref{sec:applications} we give applications of our approach.

\section{Preliminaries}
\label{sec:prelims}
 For an integer $k$, let $[k]$ denote the set $\{1,2,\ldots,k\}$. For a vector $v \in \R^k$, we use the following notation for several standard quantities: $\|v\|_0 = |\{i \in [k] \cond v_i \neq 0\}|$, $\|v\|_1 = \sum_{i\in [k]}|v_i|$, $\|v\|_\infty = \max_{i\in [k]}\{|v_i|\}$ and $\|v\|_2 = \sqrt{\sum_{i\in [k]}v_i^2}$.
 For a real value $\alpha$, we denote its projection to $[-1,1]$ by $P_1(\alpha)$. That is, $P_1(\alpha) = \alpha$ if $|\alpha| \leq 1$ and $P_1(\alpha) = \sgn(\alpha)$, otherwise.

We refer to real-valued functions with range in $[-1,1]$ as {\em bounded}.
Let $B_d = \{a \in \zon \cond \|a\|_0 \leq d\}$. For $a \in \zon$ let $\chi_a(x)$ denote the function $\prod_{a_i = 1} x_i$. It is a monomial and also a parity function over variables with indices in $\{ i \leq n \cond a_i =1\}$. A degree-$d$ polynomial threshold function is a function representable as $\sgn(\sum_{a \in B_d} w(a) \chi_a(x))$ for some vector of weights $w \in \R^{B_d}$.
When the representing vector $w$ is sparse we can describe it by listing all the non-zero coefficients only. We refer to this as being {\em succinctly represented}.

\noindent {\bf PAC learning. }
Our learning model is Valiant's \citeyearpar{Valiant:84} well-known PAC model. In this model, for a concept $f$ and distribution $D$ over $\pmi^n$, an {\em example oracle} $\exoracle{D}{f}$ is an oracle that, upon request, returns an example $(x,f(x))$
where $x$ is chosen randomly with respect to $D$, independently of any
previous examples. A {\em membership query} (MQ) learning algorithm is an algorithm that has oracle access to the target function $f$ in addition to $\exoracle{D}{f}$, namely it can, for every point $x \in \pmi^n$ obtain the value $f(x)$. For $\eps \geq 0$, we say that function $g$ is
$\eps$-close to function $f$ relative to distribution $D$
if $\pr_D[f(x)=g(x)]\geq 1-\eps$. For a concept class $C$, we say that an algorithm $\A$
{\em efficiently} learns $C$ over distribution $D$, if for every $\eps > 0$,
 $n$, $f \in C$, $\A$ outputs, with probability at least $1/2$ and in time polynomial in $n/\eps$, a hypothesis $h$ that is $\eps$-close to $f$ relative to $D$. Learning of DNF expressions is commonly parameterized by the size $s$ (\ie the number of terms) of the smallest-size DNF representation of $f$. In this case the running time of the efficient learning algorithm is also allowed to depend polynomially on $s$. For $k \in [n]$ an $s$-term $k$-DNF expression is a DNF expression with $s$ terms of length at most $k$.

\noindent {\bf Fourier transform. } A number of methods for learning over the uniform distribution $\U$ are based on the Fourier transform technique. The technique relies on the fact that the set of all parity functions $\{\chi_a(x)\}_{a \in \zon}$ forms an orthonormal basis of the linear space of real-valued function over $\pmi^n$ with inner product defined as $\la f,g\ra_\U = \E_\U[f(x)g(x)]$. This fact implies that any real-valued  function $f$ over $\pmi^n$ can be uniquely represented as a linear combination of parities, that is $f(x) = \sum_{a \in \zo^n} \hat{f}(a) \chi_a(x)$. The coefficient $\hat{f}(a)$ is called Fourier coefficient of $f$ on $a$ and equals $\E_{\U}[f(x)\chi_a(x)]$;  $\|a\|_0$ is called the {\em degree} of $\hat{f}(a)$.
For a set $S \subseteq \zon$ we use $\hat{f}(S)$ to denote the vector of all coefficients with indices in $S$ and $\hat{f}$ to denote the vector of all the Fourier coefficients of $f$. The vector of all degree-$(\leq d)$ Fourier coefficients of $f$ can then be expressed as $\hat{f}(B_d)$. We also use a similar notation for vectors of estimates of Fourier coefficients. Namely, for $S \subseteq \zon$ we use $\tilde{f}(S)$ to denote a vector in $\R^S$ indexed by vectors in $S$.
We denote by $\tilde{f}(a)$ the $a$-th element of $\tilde{f}(S)$.  Whenever appropriate, we use succinct representations for vectors of Fourier coefficients (\ie listing only the non-zero coefficients).

We will make use of Parseval's identity which states that for every real-valued function $f$ over $\pmi^n$, $\E_\U[f^2] = \sum_a \hat{f}(a)^2 = \|\hat{f}\|_2^2$. Given oracle access to a function $f$ (\ie membership queries), the Fourier transform of a function can be approximated using the KM algorithm \citep{GoldreichLevin:89,KushilevitzMansour:93}
\begin{theorem}[KM algorithm]
\label{th:km}
There exists an algorithm that for any real-valued function $f :\pmi^n \rightarrow [-1,1]$, given parameters $\theta>0$, $\delta>0$ and oracle access to $f$, with probability at least $1-\delta$, returns a succinctly represented vector $\tilde{f}$, such that $\|\hat{f} - \tilde{f}\|_\infty \leq \theta$ and $\|\tilde{f}\|_0 \leq 4/\theta^2$. The algorithm runs in $\tilde{O}(n^2 \cdot \theta^{-6} \cdot \logd)$ time and makes $\tilde{O}(n \cdot \theta^{-6} \cdot \logd)$ queries to $f$.
\end{theorem}

\noindent {\bf Product distributions.}
We consider learning over product distributions on $\pmi^n$. For a vector $\mu \in (-1,1)^n$ let $D_\mu$ denote the product distribution over $\pmi^n$ such that $\E_{x \sim D_\mu}[x_i] = \mu_i$ for  every $i \in [n]$. For each $i\in[n]$, $x_i = 1$ with probability $(1+\mu_i)/2$. For $c \in (0,1]$ the distribution $D_\mu$ is said to be $c$-bounded if $\mu \in [-1+c,1-c]^n$.  The uniform distribution is then equivalent to $D_{\bar{0}}$, where $\bar{0}$ is the all-zero vector, and is $1$-bounded. We use $\E_\mu[\cdot]$ to denote $\E_{x \sim D_\mu}[\cdot]$ and $\E[\cdot]$ to denote $\E_{x \sim \U}[\cdot]$ and similarly for $\pr$.

The Fourier transform technique extends naturally to product distributions \citep{FurstJS:91}. For $\mu \in (-1,1)^n$ the inner product is defined as
 $\la f,g\ra_\mu = \E_\mu[f(x)g(x)]$. The corresponding orthonormal basis of functions over $D_\mu$ is given by the set of functions $\{\phi_{\mu,a} \cond a \in \zon\}$, where $\phi_{\mu,a}(x) = \prod_{a_i=1} \frac{x_i-\mu_i}{\sqrt{1-\mu_i^2}}.$
Every function $f: \pmi^n \rightarrow \R$ can be uniquely represented as
$f(x) = \sum_{a \in \zo^n} \hat{f}_\mu(a) \phi_{\mu,a}(x)$, where the $\mu$-Fourier coefficient $\hat{f}_\mu(a)$ equals $\E_\mu[f(x)\phi_{\mu,a}(x)]$. We extend our uniform-distribution notation for vectors of Fourier coefficients to product distributions analogously. For any product distribution $\mu$, a degree-$d$ polynomial $p(x)$ has no non-zero $\mu$-Fourier coefficients of degree greater than $d$.

The KM algorithm has been extended to product distributions by \citet{Bellare:91}  \citep[see also][]{Jackson:97}. Below we describe a more efficient version given by \citet{KalaiST:09} (referred to as the EKM algorithm) which is efficient for all product distributions.
\begin{theorem}[EKM algorithm]
\label{th:ekm}
 There exists an algorithm that for any real-valued function $f :\pmi^n \rightarrow [-1,1]$, given parameters $\theta>0$, $\delta>0$, $\mu \in (-1,1)^n$, and oracle access to $f$, with probability at least $1-\delta$, returns a succinctly represented vector $\tilde{f}_\mu$, such that $\|\hat{f}_\mu - \tilde{f}_\mu\|_\infty \leq \theta$ and $\|\tilde{f}_\mu\|_0 \leq 4/\theta^2$. The algorithm runs in time polynomial in $n$, $1/\theta$ and $\logd$.
\end{theorem}
When learning relative to distribution $D_\mu$ we can assume that $\mu$ is known to the learning algorithm. For our purposes a sufficiently-close approximation to $\mu$ can always be obtained by estimating $\mu_i$ for each $i$ using random samples from $D_\mu$.

Without oracle access to $f$, but given examples of $f$ on points drawn randomly from $D_\mu$ one can estimate the Fourier coefficients up to degree $d$ by estimating each coefficient individually in a straightforward way (that is, by using the empirical estimates). A na\"{i}ve way of analyzing the number of samples required to achieve certain accuracy requires a number of samples that depends on $\mu$ and the degree of the estimated coefficient (since $|\phi_{\mu,a}(x)|$ depends on them). \citet{KalaiST:09} gave a more refined analysis which eliminates the dependence on $d$ and $\mu$ and implies the following theorem.
\begin{theorem}[Low Degree Algorithm]
\label{th:low-degree}
There exists an algorithm that for any real-valued function $f :\pmi^n \rightarrow [-1,1]$ and $\mu \in (-1,1)^n$, given parameters $d \in [n]$, $\theta>0$, $\delta>0$, and access to $\exoracle{D_\mu}{f}$, with probability at least $1-\delta$, returns a succinctly-represented vector $\tilde{f}_\mu$, such that $\|\hat{f}_\mu(B_d) - \tilde{f}_\mu(B_d)\|_\infty \leq \theta$ and $\|\tilde{f}_\mu\|_0 \leq 4/\theta^2$. The algorithm runs in time $n^{d} \cdot \poly(n \cdot \theta^{-1} \cdot \logd)$.
\end{theorem}

\section{Structural Conditions for Approximation}
\label{sec:structural}
In this section we prove several connections relating the $L_1$ distance of a low-degree PTF $f$ to a bounded function $g$ (\ie $\E[|f(x)-g(x)|])$ and the maximum distance between the low-degree portions of the Fourier spectrum of $f$ and $g$ (\ie $\|\hat{f}(B_d)-\hat{g}(B_d)\|_\infty$). A special case of such a connection was proved by \citet{KalaiST:09}. Another special case, for linear threshold functions, was given by \citet{BirkendorfDJKS:98}. Our version yields strong bounds for every PTF $f(x)=\sgn(p(x))$ where polynomial $p(x)$ satisfies $|p(x)|\geq 1$ for all $x$ and $p(x)$ is close to a low-degree polynomial $p'(x)$ of small $\|\cdot\|_1$ norm. In particular, it applies to any function representable as an integer-weight low-degree PTF of polynomial total weight and to any integer-weight linear threshold of terms (ANDs) of polynomial total weight (which includes polynomial size DNF expressions). We start by defining two simple and known measures of complexity of a degree-$d$ PTF.

\begin{definition}
For $\lambda>0$, we say that a polynomial $p(x)$, {\em $\lambda$-sign-represents} a boolean function $f(x)$ if for all $x \in \pmn$, $f(x) = \sgn(p(x))$ and $|p(x)| \geq \lambda$. For a degree-$d$ PTF $f$, let $W_1^d(f)$ denote $$\min\{\|\hat{p}\|_1 \cond p \mbox{ 1-sign-represents } f\} .$$ The {\em degree-$d$ total integer weight} of $f$ is $$TW^d(f) =  \min\{\|\hat{p}\|_1 \cond \hat{p} \mbox{ is integer and } f = \sgn(p)\} .$$
\end{definition}

\begin{remark}
\label{rem:ptf-margin-to-weight}
We briefly remark that $W_1^d(f)$ is exactly the inverse of the {\em advantage} of a degree-$d$ PTF defined by \citet{KrausePudlak:97} as the largest $\lambda$ for which there exists a polynomial $p(x)$ such that $p$ $\lambda$-sign-represents $f$ and $\|\hat{p}\|_1 =1$). In addition, linear programming duality implies that the advantage of $f$ equals $\alpha$ if and only if $\alpha$ is the smallest value such that for every distribution $D$ over $\pmn$ there exists a monomial $\chi_a(x)$ of degree at most $d$ such that $|\E_D[f(x) \cdot \chi_a(x)]| \geq \alpha$ (see Nisan's proof in \citep{Impagliazzo:95}). Finally, clearly $W_1^d(f) \leq TW^d(f)$. The characterization of advantage using the LP duality together with the boosting algorithm by \citet{Freund:95} imply that $TW^d(f) = O(n \cdot W_1^d(f)^2)$.
\end{remark}

We first prove a simpler special case of our bound when the representing polynomial $p(x)$ and the approximating polynomial $p'(x)$ are the same.
\begin{lemma}
\label{lem:ptf-fourier-bound}
 Let $p(x)$ be a degree-$d$ polynomial that 1-sign-represents a PTF $f(x)$. For every $\mu \in (-1,1)^n$ and bounded function $g(x):\pmi^n\rightarrow [-1,1]$, $$\E_\mu[|f(x)-g(x)|] \leq \|\hat{f}_\mu(B_d) - \hat{g}_\mu(B_d)\|_\infty \cdot \|\hat{p}_\mu(B_d)\|_1 .$$
\end{lemma}
\begin{proof}
First note that for every $x$, the values $f(x), f(x)-g(x)$ and $p(x)$ have the same sign. Therefore
$\E_\mu[|f(x)-g(x)|] = \E_\mu[f(x)(f(x)-g(x))] \leq \E_\mu[p(x)(f(x)-g(x))]$.
From here we immediately get that
\alequn{\E_\mu[p(x)(f(x)-g(x))] &= \sum_{a \in B_d} \hat{p}_\mu(a)\E_\mu[(f(x)-g(x)) \phi_{\mu,a}(x)] = \sum_{a \in B_d} \hat{p}_\mu(a)(\hat{f}_\mu(a) - \hat{g}_\mu(a)) \\ &\leq \|\hat{f}_\mu(B_d) - \hat{g}_\mu(B_d)\|_\infty \cdot \|\hat{p}_\mu(B_d)\|_1\ .}
\end{proof}

To apply our bound to functions which are close (but not equal) to a degree-$d$ PTF we also give the following approximate version of Lemma \ref{lem:ptf-fourier-bound}.
\begin{lemma}
\label{lem:ptf-fourier-approx-bound}
Let $p(x)$ be a polynomial that 1-sign-represents a PTF $f(x)$ and let $p'(x)$ be any degree-$d$ polynomial. For every $\mu \in (-1,1)^n$ and a bounded function $g(x):\pmi^n\rightarrow [-1,1]$, $$\E_\mu[|f(x)-g(x)|] \leq \|\hat{f}_\mu(B_d) - \hat{g}_\mu(B_d)\|_\infty \cdot \|\widehat{p'}_\mu(B_d)\|_1 + 2 \E_\mu[|p'(x) - p(x)|] .$$
\end{lemma}
\begin{proof}
Following the proof of Lemma \ref{lem:ptf-fourier-approx-bound}, we get
\alequn{\E_\mu[|f(x)-g(x)|] & \leq \E_\mu[p(x)(f(x)-g(x))] \\ &= \E_\mu[p'(x)(f(x)-g(x))] + \E_\mu[(p(x)-p'(x))(f(x)-g(x))] \\ & \leq \|\hat{f}_\mu(B_d) - \hat{g}_\mu(B_d)\|_\infty \cdot \|\widehat{p'}_\mu(B_d)\|_1\  + \E_\mu[2|p'(x) - p(x)|].}
\end{proof}

We now give bounds on such representations of DNF expressions.  As a warm-up we start with the uniform distribution case which is implicit in \citep{KalaiST:09}.
\begin{lemma}
\label{lem:dnf-l1-bound}
For any $s$-term DNF $f$, $W_1^n(f) \leq 2s+1$.
\end{lemma}
\begin{proof}
Let $t_1(x),t_2(x),\ldots,t_s(x)$ denote the $\zo$ versions of each of the terms of $f$. For each $i\in [s]$ let $T_i$ denote the set of the indices of all the variables in the term $t_i$. Then,  $t_i = \prod_{j \in T_i}\frac{1\pm x_j}{2}$, where the sign of each variable $x_j$ is determined by whether it is negated or not in $t_i$. As is well-known \citep[\eg][]{BlumFJ+:94}, this implies that $\|\hat{t_i}\|_1 = 1$. Now, let $p(x) = 2 \sum_{i\in[s]} t_i(x) - 1$. It is easy to see that, $|p(x)| \geq 1$, $f(x) = \sgn(p(x))$, $p(x)$ and $$\|\hat{p}\|_1 \leq 2 \sum_{i\in[s]} \|\hat{t_i}\|_1 + 1 \leq 2s+1\ .$$
\end{proof}

An immediate corollary of Lemma \ref{lem:ptf-fourier-bound} and Lemma \ref{lem:dnf-l1-bound} is the following bound given by \citet{KalaiST:09}.
\begin{corollary}
\label{cor:uniform-dnf-fourier-approx-bound}
Let $f$ be an $s$-term DNF expression. For every bounded function $g(x)$, $\E[|f(x)-g(x)|] \leq (2s+1)\cdot \|\hat{f} - \hat{g}\|_\infty$.
\end{corollary}

As can be seen from of Lemma \ref{lem:dnf-l1-bound}, bounding $W_1^n(f)$ is based on bounding $\|\hat{t_i}\|_1$ for every term $t_i$ of a DNF expression. Therefore we next prove a product distribution bound on $\|\hat{t_i}\|_1$.
\begin{lemma}
\label{lem:term-l1-bound-prod}
Let $t(x)$ be a $\zo$ AND of $d$ boolean literals, that is, for a set of $d$ literals $T \subseteq \{x_1,\bar{x}_1,x_2,\bar{x}_2,\ldots,x_n,\bar{x}_n\}$, $t(x) = 1$ when all literals in $T$ are set to 1 in $x$ and 0 otherwise.
For any constant $c \in (0,1]$ and $\mu \in [-1+c,1-c]^n$, $$\|\hat{t}_\mu\|_1 = \|\hat{t}_\mu(B_d)\|_1 \leq (2-c)^{d/2} .$$
\end{lemma}
\begin{proof}
Let $S$ denote the set of all vectors in $\zon$ corresponding to subsets of $T$, that is $$S=\{ a \cond \forall i\in[n],\ (a_i=0 \bigvee \{x_i,\bar{x}_i\} \cap T \neq \emptyset)\}.$$

Clearly, $\|\hat{t}_{\mu}\|_1 = \|\hat{t}_{\mu}(B_d)\|_1 = \|\hat{t}_{\mu}(S)\|_1$. In addition, by Parseval's identity $$\|\hat{t}_\mu\|_2^2 = \E_\mu[t(x)^2]= \pr_\mu[t(x)=1]\leq (1-c/2)^d\ .$$
Now, by the Cauchy-Schwartz inequality,
$$\|\hat{t}_{\mu}(S)\|_1 \leq 2^{d/2} \cdot \|\hat{t}_{\mu}\|_2 = 2^{d/2} \cdot (1-c/2)^{d/2} = (2-c)^{d/2}\ ,$$
giving us the desired bound.
\end{proof}

We now use Lemmas \ref{lem:ptf-fourier-approx-bound}  and \ref{lem:term-l1-bound-prod} to give a bound for all product distributions.


\begin{theorem}
\label{th:dnf-fourier-approx-bound}
 Let $c \in (0,1]$ be a constant, $\mu$ be a $c$-bounded distribution and $\eps > 0$. For an integer $s > 0$ let $f$ be an $s$-term DNF.
For $d = \lfloor \log{(s/\eps)}/\log{(2/(2-c))} \rfloor$ and every bounded function $g:\pmi^n \rightarrow [-1,1]$, $$\E_\mu[|f(x)-g(x)|] \leq (2 \cdot (2-c)^{d/2} \cdot s + 1 ) \cdot \|\hat{f}_\mu(B_d) - \hat{g}_\mu(B_d)\|_\infty + 4\eps .$$
\end{theorem}
\begin{proof}
As in the proof of Lemma \ref{lem:dnf-l1-bound}, let $t_1(x),t_2(x),\ldots,t_s(x)$ denote the $\zo$ versions of each of the terms of $f$ and let $p(x) = 2 \sum_{i\in[s]} t_i(x) - 1$ be a polynomial that 1-sign-represents $f$.  Now let $M \subseteq [s]$ denote the set of indices of $f$'s terms which have length $\geq d+1 \geq \log{(s/\eps)}/\log{(2/(2-c))}$ and let $p'(x) = 2\sum_{i\not\in M} t_i(x) -1$. In other words, $p'$ is $p$ with contributions of long terms removed and, in particular, is a degree-$d$ polynomial.

For each $i \in M$, $\E_\mu[t_i(x)] = \pr_\mu[t_i(x)=1] \leq (1-c/2)^{d+1} \leq \eps/s$. This implies that \equ{\E_\mu[|p'(x) - p(x)|] \leq \sum_{i\in M} \E_\mu[2|t_i(x)|] \leq 2\eps\ .\label{eq:dnf-small-poly-diff}}  Using Lemma \ref{lem:term-l1-bound-prod}, we get \equ{\|\widehat{p'}_\mu(B_d)\|_1 \leq 2\sum_{i\not\in M} \|\hat{t_i}_\mu(B_d)\|_1 + 1  \leq 2\cdot (2-c)^{d/2}\cdot s + 1 . \label{eq:l1-dnf-total-bound-prod}}

We can now apply Lemma \ref{lem:ptf-fourier-approx-bound} and equations (\ref{eq:dnf-small-poly-diff}, \ref{eq:l1-dnf-total-bound-prod}) to obtain \alequn{\E_\mu[|f(x)-g(x)|] & \leq \|\hat{f}_\mu(B_d) - \hat{g}_\mu(B_d)\|_\infty \cdot \|\widehat{p'}_\mu(B_d)\|_1\  + 2\E_\mu[|p'(x) - p(x)|] \\ & \leq (2\cdot (2-c)^{d/2}\cdot s +1) \cdot \|\hat{f}_\mu(B_d) - \hat{h}_\mu(B_d)\|_\infty  + 4\eps .}
\end{proof}

It is easy to see that Theorem \ref{th:dnf-fourier-approx-bound} generalizes to any function that can be expressed as low-weight linear threshold of terms. Specifically, we prove the following generalization (the proof appears in Appendix \ref{sec:app-proofs}).
\begin{theorem}
\label{th:ltf-fourier-approx-bound}
 Let $c \in (0,1]$ be a constant, $\mu$ be a $c$-bounded distribution and $\eps > 0$. For an integer $s > 0$ let $f = h(u_1,u_2,\ldots,u_s)$, where $h$ is an LTF over $\pmi^s$ and $u_i$'s are terms.
For $d = \lfloor \log{(W_1^1(h)/\eps)}/\log{(2/(2-c))} \rfloor$ and every bounded function $g:\pmi^n \rightarrow [-1,1]$, $$\E_\mu[|f(x)-g(x)|] \leq (2 \cdot (2-c)^{d/2}  +1)\cdot W_1^1(h) \cdot \|\hat{f}_\mu(B_d) - \hat{g}_\mu(B_d)\|_\infty + 4\eps .$$ For $c=1$, $(2-c)^{d/2} =1$ and for $c \in (0,1)$, $(2-c)^{d/2} \leq (W_1^1(h)/\eps)^{(1/\log{(2/(2-c))} -1)/2}$.
\end{theorem}

\section{Construction of a Fourier Spectrum Approximating Function}
\label{sec:construct}
As follows from Corollary \ref{cor:uniform-dnf-fourier-approx-bound} (and Th.~\ref{th:dnf-fourier-approx-bound}), to $\eps$-approximate a DNF expression over a product distribution, it is sufficient to find a bounded function $g$ such that $g$ has approximately the same Fourier spectrum as $f$. In this section we show how this can be done by giving an algorithm which constructs a function with the desired Fourier spectrum or the low-degree part thereof.

Our algorithm is based on the following idea: given a bounded function $g$ such that for some $a$, $|\hat{f}(a) - \hat{g}(a)| \geq \gamma$ we show how to obtain a bounded function $g_1$ which is closer in $L_2$ distance squared to $f$ than $g$. Parseval's identity states that $\E[(f-g)^2] = \sum_b (\hat{f}(b)-\hat{g}(b))^2$. Therefore to improve the distance to $f$ we do the simplest imaginable update: define $g' = g + (\hat{f}(a) - \hat{g}(a)) \chi_a$. In other words $g'$ is the same as $g$ but with $a$'s Fourier coefficient set to $\hat{f}(a)$. Clearly,
$$\E[(f-g')^2] = \sum_{b\neq a} (\hat{f}(b)-\hat{g}(b))^2 = \E[(f-g)^2] - (\hat{f}(a) - \hat{g}(a))^2  \leq \E[(f-g)^2] - \gamma^2. $$
The only problem with this approach is that $g'$ is not necessarily a function with values bounded in $[-1,1]$. However, following the idea from \citep{Feldman:09sqd}, we can we convert $g'$ to a bounded function $g_1$ by cutting-off all values outside of $[-1,1]$ (which is achieved by applying the projection function $P_1$). The target function $f$ is boolean and therefore this step can only decrease the $L_2$ distance squared to $f$. This simple argument implies that starting with $g \equiv 0$ we can update it iteratively until we reach a bounded function $g_t$ such that for all $a$, $|\hat{f}(a) - \hat{g}(a)| \leq \gamma$. The decrease in the $L_2$ distance squared at every step implies that the total number of steps cannot exceed $1/\gamma^2$. Also note that for running this algorithm the only thing we need are (the approximate values of) the Fourier coefficients of $f$.

\eat{
is equivalent to $|\E[(f-g)\chi_a]| \geq \gamma$. This means that $\chi_a$ is correlated with function $f-g$. The function $-2(f-g)$ is the gradient of the function $(f(x)-g(x))^2$ at point (in the $2^n$ dimensional space) $g(x)$. Therefore, as was observed in \citep{TrevisanTV:09} and \citep{Feldman:09sqd}, for $\gamma' = \sgn(\hat{f}(a) - \hat{g}(a)) \gamma$ and  $g' = g + \gamma' \cdot \chi_a$ we will obtain that $\E[(f-g'^2] \leq \E[(f-g)^2] - \gamma^2$. In other words, $g'$ is closer in $L_2$ distance squared to $f$ than $g$.
}

We now state and prove the claim formally. The input to our algorithm is a vector $\tilde{f}(B_d) \in \R^{B_d}$ of desired coefficients up to degree $d$ given to some accuracy $\gamma$. Further, in our applications we will only use vectors with at most $O(1/\gamma^2)$ non-zero coefficients since for every Boolean function at most $1/\gamma^2$ of its Fourier coefficients are of magnitude greater than $\gamma$ and smaller coefficients are approximated by $0$.

\begin{theorem}
\label{th:ptf-reconstruct}
There exists a randomized algorithm \ptfapprox that for every boolean function $f: \{-1,1\}^n \rightarrow \{-1,1\}$, given $\gamma>0,\delta>0$ a degree bound $d$ and a succinctly-represented vector of coefficients $\tilde{f}(B_d) \in \R^{B_d}$ such that $\|\hat{f}(B_d) - \tilde{f}(B_d)\|_\infty \leq \gamma$ and $\|\tilde{f}(B_d)\|_0 = O(1/\gamma^2)$,
with probability at least $1-\delta$, outputs a bounded function $g:\pmi^n \rightarrow [-1,1]$ such that $\|\hat{f}(B_d) - \hat{g}(B_d)\|_\infty \leq 5\gamma$. The algorithm runs in time polynomial in $n$, $1/\gamma$ and $\logd$.
\end{theorem}
\begin{proof}
We build $g$ via the following iterative process.
Let $g_0 \equiv 0$.  At step $t$, given $g_t$, we run the KM algorithm (Th.~\ref{th:km}) to compute all the Fourier coefficients of $g_t$ which are of degree at most $d$ to accuracy $\gamma/2$. Let $\widetilde{g_t}(B_d) \in \R^{B_d}$ denote the vector of estimates output by the algorithm. By Theorem \ref{th:km}, there are at most $16/\gamma^2$ non-zero coefficients in $\widetilde{g_t}(B_d)$. For now let's assume that the output of the KM is always correct; we will deal with the confidence bounds later in the standard manner.

If $\|\widetilde{g_t}(B_d) - \tilde{f}(B_d)\|_\infty \leq 7\gamma/2$, then we stop and output $g_t$.
By triangle inequality, \alequn{\|\hat{f}(B_d) - \widehat{g_t}(B_d)\|_\infty &\leq \|\hat{f}(B_d) - \tilde{f}(B_d)\|_\infty + \|\tilde{f}(B_d) - \widetilde{g_t}(B_d)\|_\infty + \|\widetilde{g_t}(B_d) - \widehat{g_t}(B_d)\|_\infty \\&\leq \gamma+7\gamma/2+\gamma/2 = 5\gamma\ ,} in other words $g_t$ satisfies the claimed condition.

Otherwise, there exists $a \in B_d$ such that $|\widetilde{g_t}(a) - \tilde{f}(a)| > 7\gamma/2$. We note that using the succinct representation of $\hat{f}(B_d)$ and $\widehat{g_t}(B_d)$ such $a$ can be found in $O(n (\|\widetilde{g_t}\|_0 + \|\tilde{f}\|_0)) = O(n/\gamma^2)$ time. First observe that, by triangle inequality, $$|\widehat{g_t}(a) - \hat{f}(a)| \geq |\widetilde{g_t}(a) - \tilde{f}(a)| -  |\tilde{f}(a) - \hat{f}(a)| - |\widehat{g_t}(a) - \widetilde{g_t}(a)| \leq 7\gamma/2 - \gamma - \gamma/2 = 2\gamma.$$

Let $g'_{t+1} = g_t +  (\tilde{f}(a) - \widetilde{g_t}(a))\chi_a$. The Fourier spectrums of $g_t$ and $g'_{t+1}$ differ only on $a$. Therefore, by using Parseval's identity, we obtain that \alequ{\E[(f-g_t)^2] - \E[(f-g'_{t+1})^2] \nonumber &= (\hat{f}(a)-\widehat{g_t}(a))^2 - (\hat{f}(a) -\tilde{f}(a)+ \widetilde{g_t}(a)-\hat{g}(a))^2 \\ &\geq (2\gamma)^2 - (3\gamma/2)^2 =7\gamma^2/4\ . \label{eq:decrement}}
Now let $g_{t+1} = P_1(g_t)$. For every $x$, $(f(x)-g_{t+1}(x))^2 \leq (f(x)-g'_{t+1}(x))^2$. Together with equation (\ref{eq:decrement}) this implies that $\E[(f-g_{t+1})^2] \leq  \E[(f-g_t)^2] - 7\gamma^2/4$.
At step $0$ we have $\E[(f-g_0)^2] = 1$ and therefore the process will terminate after at most $4/(7\gamma^2)$ steps.

We note that in order to make sure that the success probability is at leat $1-\delta$ it is sufficient to run the KM algorithm with confidence parameter $4\delta /(7\gamma^2)$. At step $t$ evaluating $g_t$ on any point $x$ takes $O(t \cdot n)$ time and therefore each invocation of the KM algorithm takes $\tilde{O}(n^2 \cdot\gamma^{-8} \cdot \logd)$ time. Overall this implies that the running time of \ptfapprox is $\tilde{O}(n^2 \cdot\gamma^{-10} \cdot \logd)$.
\end{proof}

A simple observation about \ptfapprox is that it does not rely on the update step being a multiple of a boolean function. Therefore it would work verbatim for any orthonormal basis and not only parities. Therefore, by using the EKM algorithm in place of KM we can easily extend our algorithm to any product distribution.
\begin{theorem}
\label{th:ptf-reconstruct-product}
There exists a randomized algorithm \ptfapproxprod that for every $\mu \in (-1,1)^n$, boolean function $f: \{-1,1\}^n \rightarrow \{-1,1\}$, given $\mu, \gamma>0,\delta>0$, a degree bound $d$ and a succinctly-represented vector of coefficients $\tilde{f}_\mu(B_d) \in \R^{B_d}$ such that $\|\hat{f}_\mu(B_d) - \tilde{f}_\mu(B_d)\|_\infty \leq \gamma$ and $\|\tilde{f}_\mu(B_d)\|_0 = O(1/\gamma^2)$, with probability at least $1-\delta$, outputs a function $g:\pmi^n \rightarrow [-1,1]$ such that $\|\hat{f}_\mu(B_d) - \hat{g}_\mu(B_d)\|_\infty \leq 5\gamma$. The algorithm runs in time polynomial in $n$, $1/\gamma$ and $\logd$.
\end{theorem}

\subsection{A Proper Construction Algorithm}
\label{sec:proper-recon}
One disadvantage of this construction is that $g$ output by \ptfapprox is not a PTF itself. The reason for this is that the projection operation $P_1$ is applied after every update. We now show that instead of applying the projection step after every update it is sufficient to apply the projection once to all the updates. This idea is based on Impagliazzo's \citeyearpar{Impagliazzo:95} argument in the context of hardcore set construction, and is also the basis for the algorithm of \citet{TrevisanTV:09}. Impagliazzo's proof uses the same $L_2$ squared potential function but requires an additional point-wise counting argument to prove that the potential can be used to bound the number of steps. Instead, we augment the potential function in a way that captures the additional counting argument and generalized to non-boolean functions (necessary for the product distribution case). As a result the algorithm will output a function of the form $P_1(\sum_{a\in B_d} \alpha_a \chi_a)$ which is then converted to a PTF by applying the sign function. The same idea is also used in the Chow parameter reconstruction algorithm of \citet{DeDFS:12}. The modified proof also allows us to easily derive a bound on the total integer weight of the resulting PTF and optimize the running time of the algorithm (the optimization of running time is deferred to a full version of this work).

\eat{
We will use the following definition for the type of functions output by our proper construction algorithm.
\begin{definition} \label{def:lbf}
A function $g:\{-1,1\}^n \rightarrow [-1,1]$ is referred to as a degree $d$ {\em polynomial bounded function} (PBF)  if there exists a vector of real values $w \in \R^{B_d}$ such that
$g(x)  = P_1(\sum_{a \in B_d}^n w(a) \chi_a(x))$. The vector $w$ is said to represent $g$.
\end{definition}
}

\begin{theorem}
\label{th:ptf-reconstruct-proper}
There exists a randomized algorithm \ptfreconprod that for every $\mu \in (-1,1)^n$, boolean function $f: \{-1,1\}^n \rightarrow \{-1,1\}$, given $\mu, \gamma>0,\delta>0$, a degree bound $d$ and a succinctly-represented vector of coefficients $\tilde{f}_\mu(B_d) \in \R^{B_d}$ such that $\|\hat{f}_\mu(B_d) - \tilde{f}_\mu(B_d)\|_\infty \leq \gamma$ and $\|\tilde{f}_\mu(B_d)\|_0 = O(1/\gamma^2)$,
with probability at least $1-\delta$, outputs a bounded function $g:\pmi^n \rightarrow [-1,1]$ such that $\|\hat{f}_\mu(B_d) - \hat{g}_\mu(B_d)\|_\infty \leq 5\gamma$. The algorithm runs in time polynomial in $n$, $1/\gamma$ and $\logd$. In addition, $g(x) = P_1(g'(x))$ for a degree-$d$ polynomial such that $\widehat{g'}_\mu = \gamma \cdot \hat{p}_\mu$ where $\hat{p}_\mu$ is a vector of integers and $\|\hat{p}_\mu\|_1 \leq 1/(2\gamma^2)$.
\end{theorem}
\begin{proof}
As in the proof of Theorem \ref{th:ptf-reconstruct}, we build $g$ via an iterative process starting from $g'_0 \equiv 0$ and $g_0 = P_1(g'_0)$.
We use the EKM algorithm (Th.~\ref{th:ekm}) to compute $\widetilde{g_t}_\mu(B_d)$ and stop and return $g_t$ if $\|\widetilde{g_t}_\mu(B_d) - \tilde{f}_\mu(B_d)\|_\infty \leq 7\gamma/2$.
Otherwise (there exists $a \in B_d$ such that $|\widetilde{g_t}_\mu(a) - \tilde{f}_\mu(a)| > 7\gamma/2$ and $|\widehat{g_t}_{\mu}(a) - \hat{f}_\mu(a)| > 2\gamma$), we let $\gamma' = \gamma \cdot \sgn(\tilde{f}_\mu(a) - \widetilde{g_t}_\mu(a))$, $g'_{t+1} = g'_{t} + \gamma' \chi_{a,\mu}$ and $g_{t+1} = P_1(g'_{t+1})$.

We prove a bound on the total number of steps using the following potential function: $$E(t) = \E_\mu[(f-g_t)^2] + 2\E_\mu[(f-g_t)(g_t-g'_t)] = \E_\mu[(f-g_t)(f-2g'_t+g_t)].$$ The key claim of this proof is that  $E(t) - E(t+1) \geq \gamma^2$.
First,
\begin{eqnarray}
E(t) - E(t+1) & =& \E_\mu[(f-g_t)(f-2g'_t+g_t)] - \E_\mu[(f-g_{t+1})(f-2g'_{t+1}+g_{t+1})] \notag   \\
&=& \E_\mu\left[(f-g_t)(2 g'_{t+1} - 2 g'_t) - (g_{t+1}-g_t)(2g'_{t+1}-g_t-g_{t+1})        \right] \notag \\
& =& \E_\mu[2(f-g_t)\gamma' \chi_{a,\mu}] - \E_\mu\left[(g_{t+1}-g_t)(2g'_{t+1}-g_t-g_{t+1})  \right] \label{eq:bound-step}
\end{eqnarray}
We observe that $\E_\mu[2(f-g_t)\gamma' \chi_{a,\mu}] = 2 \gamma' (\hat{f}_\mu(a) - \widehat{g_t}_\mu(a))$ and that $\sgn(\hat{f}_\mu(a) - \widehat{g_t}_\mu(a)) = \sgn(\tilde{f}_\mu(a) - \widetilde{g_t}_\mu(a))$. Therefore, \eat{using equation (\ref{eq:coeff-diff})} we get \equ{\E_\mu[2(f-g_t)\gamma' \chi_a] \geq 2 \gamma |\hat{g}_{t,\mu}(a) - \hat{f}_\mu(a)| \geq 4\gamma^2\ .\label{eq:bound-gradient}}

To upper-bound the expression $\E_\mu\left[(g_{t+1}-g_t)(2g'_{t+1}-g_t-g_{t+1})\right]$ we prove that for every point $x \in \{-1,1\}^n$, \equn{(g_{t+1}(x)-g_t(x))(2g'_{t+1}(x)-g_t(x)-g_{t+1}(x)) \leq 2\gamma^2 \chi_{a,\mu}(x)^2.} We first observe that $|g_{t+1}(x)-g_t(x)| = |P_1(g'_{t}(x)+ \gamma' \chi_{a,\mu}(x)) - P_1(g'_t(x))| \leq |\gamma' \chi_{a,\mu}(x)| = |\gamma \chi_{a,\mu}(x)|$ (a projection operation does not increase the distance).
Now $$|2 g'_{t+1}(x)-g_t(x)-g_{t+1}(x)| \leq |g'_{t+1}(x)-g_t(x)| + |(g'_{t+1}(x)-g_{t+1}(x)|. $$
The first part $|g'_{t+1}(x)-g_t(x)| = |\gamma' \chi_{a,\mu}(x) + g'_t(x)-g_t(x)| \leq |\gamma'\chi_{a,\mu}(x)|$ unless $g'_t(x)- g_t(x) \neq 0$ and $g'_t(x)- g_t(x)$ has the same sign as $\gamma' \chi_{a,\mu}(x)$. However, in this case $g_{t+1}(x)=g_t(x)$ and as a result $(g_{t+1}(x)-g_t(x))(2g'_{t+1}(x)-g_t(x)-g_{t+1}(x)) =0$. Similarly, $|g'_{t+1}(x)-g_{t+1}(x)| \leq |\gamma' \chi_{a,\mu}(x)| $ unless $g_{t+1}(x)=g_t(x)$. Altogether we obtain that $$(g_{t+1}(x)-g_t(x))(2g'_{t+1}(x)-g_t(x)-g_{t+1}(x)) \leq \max\{0, |\gamma \chi_{a,\mu}(x)| ( |\gamma' \chi_{a,\mu}(x)| + |\gamma' \chi_{a,\mu}(x)|)\} = 2\gamma^2 \chi_{a,\mu}(x)^2.$$ This implies that \equ{\E_\mu\left[(g_{t+1}-g_t)(2g'_{t+1}-g_t-g_{t+1})\right] \leq 2\gamma^2\E_\mu[\chi_{a,\mu}(x)^2] = 2\gamma^2 \label{eq:bound-gs}.}

By substituting equations (\ref{eq:bound-gradient}) and (\ref{eq:bound-gs}) into equation (\ref{eq:bound-step}), we obtain the claimed decrease in the potential function $$E(t) - E(t+1) \geq 4\gamma^2 - 2\gamma^2 = 2\gamma^2.$$

We now observe that $E(t)  = \E_\mu[(f-g_t)^2] + 2\E_\mu[(f-g_t)(g_t-g'_t)] \geq 0$ for all $t$. This follows from noting that for every $x$ and $f(x) \in \{-1,1\}$, either $f(x)-P_1(g'_t(x))$ and $P_1(g'_t(x))-g'_t(x))$ have the same sign or one of them equals zero. Therefore $\E_\mu[(f-g_t)(g_t-g'_t)] \geq 0$ (and, naturally, $\E_\mu[(f-g_t)^2] \geq 0$).  It is easy to see that $E(0) = 1$ and therefore this process will stop after at most $1/(2\gamma^2)$ steps.

The claim on the representation of $g_t$ output by the algorithm follows immediately from the definition of $g_t = P_1(g'_t)$ and $g'_t$ being a sum of $t$ $\mu$-Fourier basis functions multiplied by $\pm \gamma$.
\end{proof}


\eat{
\begin{remark}
The running time
\end{remark}

Mention extension to product distributions and include in the appendix.
}
\section{Applications to Learning DNF Expressions}
\label{sec:applications}
We now give several application of our approximating algorithms to the problem of learning DNF expressions in several models of learning. Our first application is a new algorithm for learning DNF expressions using membership queries over any product distribution. In the second application we show a simple algorithm for learning DNF expressions from random examples coming from a smoothed product distribution. In the third application we give a new and faster algorithm for learning MDNF over product distributions (from random examples alone).
We describe all the applications for (M)DNF expressions. However, by using the more general Theorem \ref{th:ltf-fourier-approx-bound} in place of Theorem \ref{th:dnf-fourier-approx-bound}, we immediately get that our algorithms can be also used to learn a broader set of concept classes which includes, for examples, (monotone) majorities of terms. Previous algorithms for the second and third applications rely strongly on the term-combining function being an OR.
\subsection{Learning with Membership Queries}
\eat{
To learn DNF with membership queries we can use the KM algorithm to obtain an estimate of $\hat{f}$ within $\eps/(5(2s+1))$ in $\|\cdot\|_\infty$ norm. Then we can use \ptfapprox for $d=n$ to obtain a bounded function $g$ such that $\|\hat{f} - \hat{g}\|_\infty \leq \eps/(5(2s+1))$ and hence $\E[|f-g|] \leq \eps$ (by Corollary \ref{cor:uniform-dnf-fourier-approx-bound}). Then $\pr[f \neq \sgn(g)] \leq \E[|f-g|] \leq \eps$.
}

An immediate application of Theorem \ref{th:ptf-reconstruct-product} together with the bound in Theorem \ref{th:dnf-fourier-approx-bound} and the EKM algorithm (Th.~\ref{th:ekm}) is a simple algorithm for learning DNF over any constant-bounded product distribution.
\begin{corollary}
\label{cor:learn-dnf-prod}
Let $c \in (0,1]$ be a constant. There exists a membership query algorithm \dnflearnprod that for every $c$-bounded $\mu$, efficiently PAC learns DNF expressions over $D_\mu$.
\end{corollary}
\begin{proof}
Let $\eps' = \eps/9$ and, as defined in Th.~\ref{th:dnf-fourier-approx-bound}, let $d  = \lfloor \log{(s/\eps')}/\log{(2/(2-c))} \rfloor$ and $$\gamma = \eps'/(2 (2-c)^{d/2} s + 1)
= \Omega\left((\eps/s)^{(1/\log{(2/(2-c))} +1)/2 }\right) .$$
\dnflearnprod consists of two phases:
\begin{enumerate}

\item {\bf Collect $\gamma$-approximations to all degree-$d$ $\mu$-Fourier coefficients}. In this step we run the EKM algorithm for $f$ with parameters, $\theta=\gamma$, $\delta =1/4$ and $\mu$ to obtain a succinctly-represented $\tilde{f}_\mu(B_d)$ such that $\|\tilde{f}_\mu(B_d) - \tilde{g}_\mu(B_d)\|_\infty \leq \gamma$ (EKM returns the complete $\tilde{f}_\mu$ but we discard coefficients with degree higher than $d$).

\item {\bf Construct a bounded $g$ with the given $\mu$-Fourier spectrum}. In this step we run \ptfapproxprod on $\tilde{f}_\mu(B_d)$ with parameters $d$, $\gamma$, $\mu$ and $\delta =1/4$ to construct a bounded function $g$ such that $\|\hat{f}_\mu(B_d) - \hat{g}_\mu(B_d)\|_\infty \leq 5\gamma = 5\eps'/(2 (2-c)^{d/2} s + 1)$. Note that this step requires no access to membership queries or random examples of $f$.
\end{enumerate}
We return $\sgn(g(x))$ as our hypothesis. Overall, if both steps are successful (which happens with probability at least $1/2$) then, according to Theorem \ref{th:dnf-fourier-approx-bound}, $$\E_\mu[|f-g|] \leq \|\hat{f}_\mu(B_d) - \hat{g}_\mu(B_d)\|_\infty \cdot (2 (2-c)^{d/2} s + 1) +4 \eps' = 5\gamma \cdot (2 (2-c)^{d/2} s + 1) + 4\eps' = 9\eps' = \eps. $$ This implies $\pr_\mu[f \neq \sgn(g)] \leq \E_\mu[|f-g|] \leq \eps$.

The running time of both phases of \dnflearnprod is polynomial in $n$, and $1/\gamma$, which for any constant $c\in (0,1]$, is polynomial in $n \cdot s /\eps$.
\end{proof}

As noted in the proof, the only part of our algorithm that uses membership queries is the phase that collects Fourier coefficients of logarithmic degree. This step can also be performed using weaker forms of access to the target function, such as extended statistical queries of \citet{BshoutyFeldman:02} or examples coming from a random walk on a hypercube \cite{BshoutyMOS:05}. Hence our algorithm can be adapted to those models in a straightforward way.

\subsection{Smoothed Analysis of Learning DNF over Product Distributions}
\label{sec:smoothed}
We now describe how \ptfapproxprod can be used in the context of smoothed analysis of learning DNF over product distributions introduced by \citet{KalaiST:09}. We start with a brief description of the model.

\subsubsection{Learning from Smoothed Product Distributions}
Motivated by the seminal model of smoothed analysis by \citet{SpielmanTeng:04}, \citet{KalaiST:09} defined learning a concept class $C$ with respect to smoothed product distributions as follows. The model measures the complexity of a learning algorithm with respect to a product distribution $D_\mu$ where $\mu$ is ``perturbed" randomly. More formally, $\mu$ is chosen uniformly at random from a cube $\bar{\mu} + [-c,c]^n$ for a $2c$-bounded $\bar{\mu}$. A learning algorithm in this model must, for every $\bar{\mu}$ and $f\in C$, PAC learn $f$ over $D_\mu$ with high probability over the choice of $\mu$.
\begin{definition}[\citealt{KalaiST:09}]
Let $C$ be a concept class. An algorithm $\A$ is said to learn $C$ over smoothed product distributions if for every constant $c \in (0,1/2]$, $f\in C$, $\eps,\delta > 0$, and any $2c$-bounded $\bar{\mu}$, given access to $\exoracle{D_\mu}{f}$ for a randomly and uniformly chosen $\mu \in \bar{\mu}+[-c,c]^n$, with probability at least $1-\delta$, $\A$ outputs a hypothesis $h$, $\eps$-close to $f$ relative to $D_\mu$. The probability here is taken with respect to the random choice of $\mu$, choice of random samples from $D_\mu$ and any internal randomization of $\A$. $\A$ is said to learn efficiently if its running time is upper-bounded by a polynomial in $n/(\eps \cdot \delta)$ (and the size $s$ of $f$ if $C$ is parameterized) where the degree of the polynomial is allowed to depend on $c$.
\end{definition}

\noindent {\bf Feature Finding Algorithm.}
A key insight in the results of \citet{KalaiST:09} is that if a bounded function $f$ has a low-degree significant $\bar{\mu}$-Fourier coefficient $\hat{f}_{\bar{\mu}}(a)$, then after the perturbation $f$ will have significant $\mu$-Fourier coefficients for all $b \leq a$ (here $b \leq a$ means $b_i \leq a_i$ for all $i\in [n]$). This insight leads to a simple method for finding all the significant $\mu$-Fourier coefficients of degree $d$ in time polynomial in $2^d$ instead of $n^d$ required by the Low Degree algorithm.
\begin{theorem}[Greedy Feature Construction (GFC)\citep{KalaiST:09}]
Let $c \in (0,1/2]$ be a constant. There exists an algorithm that for every $f:\pmi^n \rightarrow [-1,1]$, $d \in [n]$, $\theta,\delta > 0$, $2c$-bounded $\bar{\mu}$, given access to $\exoracle{D_\mu}{f}$ for a randomly and uniformly chosen $\mu \in \bar{\mu}+[-c,c]^n$, with probability at least $1-\delta$, outputs a succinctly-represented vector $\tilde{f}(B_d)$ such that $\|\hat{f}_\mu(B_d) - \tilde{f}_\mu(B_d)\|_\infty \leq \theta$ and $\|\tilde{f}_\mu(B_d)\|_0 \leq 4/\theta^2$. The algorithm runs in time $O((n \cdot 2^d  /(\theta \cdot \delta))^{k(c)})$ for some constant $k(c)$ which depends only on $c$.
\end{theorem}

\subsubsection{Application of \ptfapproxprod}
The Greedy Feature Construction algorithm gives an efficient algorithm for collecting $\mu$-Fourier coefficients of logarithmic degree. The application of \ptfapproxprod in this setting is now straightforward. All that needs to be done is to replace the EKM algorithm in the coefficient collection phase of \dnflearnprod (Cor.~\ref{cor:learn-dnf-prod}) with the GFC algorithm. The coefficient collection phase of \dnflearnprod requires only coefficients of logarithmic degree in the learning parameters and therefore the resulting combination runs in polynomial time (the approximator construction phase is unchanged and still uses the EKM algorithm). Thereby we obtain a new simple proof of the following theorem from \citep{KalaiST:09}.
\begin{theorem}[\citealt{KalaiST:09}]
\label{th:smooth-learn-dnf}
DNF expressions are PAC learnable efficiently over smoothed product distributions.
\end{theorem}

\subsection{Learning Monotone DNF}
\label{sec:learn-mdnf}
We now describe our algorithm for learning monotone $s$-term DNF from random examples alone. For simplicity, we describe it for the uniform distribution, but all the ingredients that we use have their product distribution versions and hence the generalization is straightforward (we describe it in Appendix \ref{sec:app-proofs}). As pointed out earlier, our algorithm is based on Servedio's algorithm for learning monotone DNF \citep{Servedio:04mondnf}. The main idea of his algorithm is to restrict learning to influential variables alone (which for a monotone function can be efficiently identified) and then run the Low Degree algorithm \ref{th:low-degree} to approximate all the Fourier coefficients of low degree on influential variables. The sign of the resulting low-degree polynomial $p(x)$ is then used as a hypothesis. The degree that is known to be sufficient for such approximation to work was derived using a Fourier concentration bound by \citet{Mansour:95} and \citet{LinialMN:93} and equals $20 \cdot \log{(s/\eps)} \cdot \log{(1/\eps)}$.

In our algorithm, instead of just taking the sign of $p(x)$ as the hypothesis, we use \ptfapprox to produce a bounded function with the same Fourier coefficients as $p(x)$. The advantage of this approach is that the degree bound required to achieve $\eps$-accuracy using our approach is reduced to $\log{(s/\eps)}+O(1)$ (and is also significantly easier to prove than the Switching Lemma-based bound of \citet{Mansour:95}). Further, the accuracy estimation in our algorithm does not depend on $n$ the number of sufficiently influential variables does not depend on $n$. As a consequence our algorithm is attribute-efficient.

Following \citet{Servedio:04mondnf}, we rely on a well-known connection between the influence of a variable and Fourier coefficients that include that variable. Formally, for a function $f:\pmi^n \rightarrow \pmi$ and $i\in [n]$ let $f_{i,1}(x)$ and $f_{i,-1}(x)$ denote $f(x)$ with bit $i$ of the input set to $1$ and $-1$, respectively. The influence of variable $i$  over distribution $D$ is defined as $I_{D,i}(f) = \pr_D[f_{i,1}(x) \neq f_{i,-1}(x)]$. We use $I_i(f)$ to denote the influence over the uniform distribution. Let $S_i = \{a \in \zon \cond a_i=1\}$. \citet{KahnKL:88} have shown that for every $i\in[n]$, \equ{I_i(f) = \sum_{a \in S_i} \hat{f}(a)^2 = \|\hat{f}(S_i)\|_2^2. \label{eq:kkl}}
The crucial use of monotonicity is that for any monotone $f$,  $I_{D,i}(f) = (\E_D[f_{i,1}(x)] - \E_D[f_{i,-1}(x)])/2$ and hence one can estimate $\|\hat{f}(S_i)\|_2^2$ using random uniform examples of $f$. We now describe our algorithm for learning monotone DNF over the uniform distribution more formally.
\begin{theorem}
\label{th:learn-mdnf}
There exists an algorithm that PAC learns $s$-term monotone DNF expressions over the uniform distribution to accuracy $\eps$ in time $\tilde{O}(n\cdot (s \cdot \log{(s/\eps)})^{O(\log{(s/\eps)})})$.
\end{theorem}
\begin{proof}
Our algorithm is based on the same two phases as \dnflearnprod in Corollary \ref{cor:learn-dnf-prod}. Hence we set $\eps' = \eps/9$, $d  = \lfloor \log{(s/\eps')} \rfloor$ and $\gamma = \eps'/(2s+1)$.

The goal of the first phase of the algorithm is to collect $\gamma$-approximations to degree-$d$ Fourier coefficients of $f$. We do this by first finding the influential variables and then using a low-degree algorithm restricted to the influential variables.

Using equation (\ref{eq:kkl}), we can conclude that if for some variable $i$, $I_i(f) = \|\hat{f}(S_i)\|_2^2 \leq \gamma^2$, then there are no Fourier coefficients of $f$, that include variable $i$ and are greater in their magnitude than $\gamma$. We can therefore eliminate variable $i$, that is approximate all of Fourier coefficients in $S_i$ by $0$. Also, as we mentioned before, $I_i(f)$ can be estimated from random examples of $f$. We will use an estimate to accuracy $\gamma^2/3$ and exclude variable $i$ if the estimate is lower than $2\gamma^2/3$ (the straightforward details of the required confidence bounds appear in the more detailed and general proof of Theorem \ref{th:learn-mdnf-prod}).

We argue that this process will eliminate all but at most $s \cdot \log{(3s/\gamma^{2})}$ variables. This, follows from the fact that if a variable $i$ appears only in terms of length greater than $\log{(3s/\gamma^{2})}$ then it cannot be influential enough to survive the elimination condition. Over the uniform distribution, each term of length greater than $\log{(3s/\gamma^{2})}$ equals 1 with probability at most $\gamma^2/(3s)$. The value $f_{i,1}(x)$ differs from $f_{i,-1}(x)$ only if $x$ is accepted by a term that includes variable $i$. There are at most $s$ terms and therefore (for a variable $i$ that appears only in terms of length $\log{(3s/\gamma^2)}$) $$(\E[f_{i,1}(x)] - \E[f_{i,-1}(x)])/2 < s \cdot \gamma^2/(3s) = \gamma^2/3.$$ Consequently, the influence of such variable $i$ cannot be within $\gamma^2/3$ of $3\gamma^2/3$ (required to survive the elimination). Therefore at the end of the first step we will end up with variables only from terms of length at most  $\log{(3s/\gamma^{2})}$. Hence there will be at most  $s \cdot \log{(3s/\gamma^{2})}$ variables left. Let $M$ denote the set of the remaining (influential) variables.

In the second step of this phase we run the low-degree algorithm for degree $d$ and $\theta = \gamma = \eps'/(2s+1)$ restricted to the variables in $M$, and let $\tilde{f}(B_d)$ be the resulting vector of approximate Fourier coefficients (the coefficients with variables outside of $M$ are 0). By Theorem \ref{th:low-degree} and the property of our influential variables $\|\hat{f}(B_d) - \tilde{f}(B_d)\|_\infty \leq \gamma.$

We can now construct an approximating function in the same way as we did in \dnflearnprod (Cor.~\ref{cor:learn-dnf-prod}). Namely, in the third step of the algorithm we run \ptfapprox on $\tilde{f}(B_d)$ to obtain a bounded function $g$ such that $\|\hat{f}(B_d)-\hat{g}(B_d)\|_\infty \leq 5 \gamma = 5\eps'/(2s+1)$. Then, by Theorem \ref{th:dnf-fourier-approx-bound}, $$\E[|f-g|] \leq (2s+1) \|\hat{f}(B_d)-\hat{g}(B_d)\|_\infty + 4\eps' \leq (2s+1)\cdot 5\eps'/(2s+1) + 4\eps' = 9\eps' = \eps. $$ Hence $\pr[\sgn(g) \neq f]\leq \eps$.

To analyze the running time of our algorithm we note that both the first and the third steps can be done in $\tilde{O}(n) \cdot \poly(s/\eps)$ time. According to Theorem \ref{th:low-degree}, the second step can be done in $n \cdot |M|^d \cdot \poly(|M|/\gamma) = n \cdot (s \cdot \log{(s/\eps)})^{O(\log{(s/\eps)})}$ time steps. Altogether, we obtain the claimed bound on the running time.
\end{proof}

A corollary of our running time bound is that for $s$ and $\eps$ such that $s/\eps = 2^{\sqrt{\log{n}}}$, $s$-term monotone DNF are learnable to accuracy $\eps$ in polynomial time. Servedio's algorithm is only guaranteed to efficiently learn $2^{\sqrt{\log{n}}}$-term MDNF to constant accuracy.

We remark that the bound on running time can be simplified for monotone $s$-term $k$-DNF expressions. Specifically, we will obtain an algorithm running in $(s \cdot k )^{O(k)} \cdot (n/\eps)^{O(1)}$ time. This algorithm can be used to obtain fully-polynomial learning algorithms for monotone $2^{\sqrt{\log{n}}}$-term $\sqrt{\log{n}}$-DNF and other subclasses of MDNF expressions for which no fully-polynomial learning algorithms were known.

In Appendix \ref{sec:app-proofs} we give the straightforward generalization of our learning algorithm to product distributions and prove the following theorem.
\begin{theorem}
\label{th:learn-mdnf-prod}
For any constant $c\in (0,1]$ there exists an algorithm \mdnflearnprod that PAC learns $s$-term monotone DNF expressions over all $c$-bounded product distributions to accuracy $\eps$ in time $\tilde{O}(n\cdot  (s \cdot \log{(s/\eps)})^{O(\log{(s/\eps)})})$.
\end{theorem}

\section*{Acknowledgements}
 I thank Sasha Sherstov for pointing out the connection of our $W_1^d(f)$ measure of a PTF $f$ to the definition of advantage by \citet{KrausePudlak:97}.


\appendix
\section{Proofs of Some Generalizations}
\label{sec:app-proofs}
\begin{theorem}
\label{th:ltf-fourier-approx-bound-omitted}[restatement of Th.~\ref{th:ltf-fourier-approx-bound}]
 Let $c \in (0,1]$ be a constant, $\mu$ be a $c$-bounded distribution and $\eps > 0$. For an integer $s > 0$ let $f = h(u_1,u_2,\ldots,u_s)$, where $h$ is an LTF over $\pmi^s$ and $u_i$'s are terms.
For $d = \lfloor \log{(W_1^1(h)/\eps)}/\log{(2/(2-c))} \rfloor$ and every bounded function $g:\pmi^n \rightarrow [-1,1]$, $$\E_\mu[|f(x)-g(x)|] \leq (2 \cdot (2-c)^{d/2}  +1)\cdot W_1^1(h) \cdot \|\hat{f}_\mu(B_d) - \hat{g}_\mu(B_d)\|_\infty + 4\eps .$$ For $c=1$, $(2-c)^{d/2} =1$ and for $c \in (0,1)$, $(2-c)^{d/2} \leq (W_1^1(h)/\eps)^{(1/\log{(2/(2-c))} -1)/2}$.
\end{theorem}
\begin{proof}
Let $w= (w_0,w_1,\ldots, w_n)$ be the weight vector of $h$ such that the linear function $q(y) = \sum_{i\in[s]} w_i y_i + w_0$ 1-sign-represents $h(y)$ and $\|w\|_1 = W_1^1(h)$. Let $p(x) = \sum_{i\in[s]} w_i u_i(x) + w_0$. Now let $M \subseteq [s]$ denote the set of indices of $f$'s terms which have length $\geq d+1 \geq \log{(W_1^1(h)/\eps)}/\log{(2/(2-c))}$ and let $p'(x) = \sum_{i\not\in M} w_i u_i(x) + w_0 - \sum_{i\in M} w_i$. In other words, $p'$ is $p$ with each term $u_i$ for $i\in M$ replaced by constant $-1$.

For each $i \in M$, $\E_\mu[|u_i(x)+1|] = 2\pr_\mu[u_i(x)=1] \leq 2(1-c/2)^{d+1} \leq 2\eps/W_1^1(h)$. This implies that \equ{\E_\mu[|p(x) - p'(x)|] = \E_\mu\left[\left|\sum_{i\in M} w_i (u_i(x)+1)\right|\right] \leq \sum_{i\in M}|w_i| \cdot \E_\mu[|u_i(x)+1|] \leq 2\eps\ .\label{eq:small-poly-diff}}

For every $i\in M$, let $t_i(x) = u_i(x)/2+1/2$, be the $\zo$ version of term $u_i$. Lemma \ref{lem:term-l1-bound-prod} implies that \equ{\|\widehat{u_i}_\mu(B_d)\|_1 \leq 2 \|\widehat{t_i}_\mu(B_d)\|_1 + 1 \leq 2\cdot (2-c)^{d/2} + 1 . \label{eq:term-l1-bound}}
The polynomial $p'(x)$ is of degree $d$ and, using inequality (\ref{eq:term-l1-bound}), we obtain \alequ{\|\widehat{p'}_\mu(B_d)\|_1 &\leq \sum_{i\not\in M} |w_i| \cdot \|\widehat{u_i}_\mu(B_d)\|_1 + \sum_{i\in M} |w_i| + |w_0| \nonumber \\ &\leq  \sum_{i\not\in M} |w_i|\cdot 2 \cdot (2-c)^{d/2} + \sum_{i \in[s]} |w_i| + |w_0| \leq W_1^1(h) (2\cdot (2-c)^{d/2} + 1) .\label{eq:l1-total-bound}}

We can now apply Lemma \ref{lem:ptf-fourier-approx-bound}  and equations (\ref{eq:small-poly-diff}, \ref{eq:l1-total-bound}) to obtain
\alequn{\E_\mu[|f(x)-g(x)|] & \leq \|\hat{f}_\mu(B_d) - \hat{g}_\mu(B_d)\|_\infty \cdot \|\widehat{p'}_\mu(B_d)\|_1\  + 2\E_\mu[|p'(x) - p(x)|] \\ & \leq (2\cdot (2-c)^{d/2}+1) \cdot W_1^1(h) \cdot \|\hat{f}_\mu(B_d) - \hat{h}_\mu(B_d)\|_\infty  + 4\eps .}
\end{proof}

\begin{theorem}[restatement of Th.~\ref{th:learn-mdnf-prod}]
\label{th:learn-mdnf-prod-app}
For any constant $c\in (0,1]$ there exists an algorithm \mdnflearnprod that PAC learns $s$-term monotone DNF expressions over all $c$-bounded product distributions to accuracy $\eps$ in time $\tilde{O}(n\cdot  (s \cdot \log{(s/\eps)})^{O(\log{(s/\eps)})})$.
\end{theorem}
\begin{proof}
As in the proof of Theorem \ref{th:learn-mdnf}, \mdnflearnprod is based on two phases: in the first phase we collect $\mu$-Fourier coefficients of the target function $f$ using a low-degree algorithm restricted to influential variables; in the second phase we construct an approximating function given the $\mu$-Fourier spectrum.

Let $D_\mu$ denote the target $c$-bounded distribution. The identification of influential variables is based on the generalization of equation (\ref{eq:kkl}) to product distribution by \citet{BshoutyTamon:96}:
for every product distribution $\mu$ and $i\in[n]$, \equ{I_{D_\mu,i}(f) = 4\mu_i(1-\mu_i)\sum_{a \in S_i} \hat{f}_\mu(a)^2 = 4\mu_i(1-\mu_i)\|\hat{f}_\mu(S_i)\|_2^2. \label{eq:kkl-prod}}

As in \dnflearnprod, we set $\eps' = \eps/9$ and $d  = \lfloor \log{(s/\eps')}/\log{(2/(2-c))} \rfloor$ and $\gamma = \eps'/(2 (2-c)^{d/2} s + 1)
= \Omega\left((\eps/s)^{(1/\log{(2/(2-c))} +1)/2 }\right)$ (as defined in Th.~\ref{th:dnf-fourier-approx-bound}).

Let $c' = 4c(1-c)$. Using equation (\ref{eq:kkl-prod}), we can conclude that if for some variable $i$, $I_{D_\mu,i}(f) = 4\mu_i(1-\mu_i) \|\hat{f}_\mu(S_i)\|_2^2 \leq c' \gamma^2$, then there are no $\mu$-Fourier coefficients of $f$, that include variable $i$ and are greater in their magnitude than $\gamma$. We can therefore eliminate variable $i$, that is approximate all of $\mu$-Fourier coefficients in $S_i$ by $0$. By definition, for a monotone $f$, $I_{D_\mu,i}(f) =  (\E_\mu[f_{i,1}(x)] - \E_\mu[f_{i,-1}(x)])/2$ and therefore $I_{D_\mu,i}(f)$ can be estimated empirically from random examples of $f$. We estimate each $I_{D_\mu,i}(f)$ to accuracy $c' \cdot \gamma^2/3$ with confidence $1-n/6$. The standard Chernoff bounds imply that $O(\gamma^{-4} \cdot \log {n})$ examples are sufficient for this. We exclude variable $i$ if the obtained estimate is lower than $c'\cdot \gamma^2/3$.

We argue that this process will eliminate all but at most $O(s \cdot \log{(s/\eps)})$ variables. This, follows from the fact that if a variable $i$ appears only in terms of length greater than $d' = \log{(3s/(c' \cdot \gamma^2))/\log{(2/(2-c))}}$ then it cannot be influential enough to survive the elimination condition. Over a $c$-bounded distribution $D_\mu$, each term of length $> d'$ equals 1 with probability at most $(1-c/2)^{d'} < c' \cdot \gamma^2/(3s)$. The value $f_{i,1}(x)$ differs from $f_{i,-1}(x)$ only if $x$ is accepted by a term that includes variable $i$. There are at most $s$ terms and therefore (for a variable $i$ that appears only in terms of length $> d'$) $$(\E_\mu[f_{i,1}(x)] - \E_\mu[f_{i,-1}(x)])/2 < s \cdot c' \cdot \gamma^2/(3s) = c' \cdot \gamma^2/3.$$  Consequently, such a variable cannot produce an estimate within $c' \cdot \gamma^2/3$ which is at least $c' \cdot 2\gamma^2/3$ meaning that at the end of the first step we will end up with variables only from terms of length at most $d' = O(\log{(s/\eps)})$. Hence there will be at most  $O(s \cdot \log{(s/\eps)})$ variables left. Let $M$ denote the set of remaining (influential) variables.

In the second step of \mdnflearnprod we run the low-degree algorithm for degree $d$, $\theta = \gamma$ and confidence $1/6$ restricted to the variables in $M$, and let $\tilde{f}_\mu(B_d)$ be the resulting vector of approximate $\mu$-Fourier coefficients (the coefficients with variables outside of $M$ are 0). By Theorem \ref{th:low-degree}, with probability at least $5/6$, $\|\hat{f}_\mu(B_d) - \tilde{f}_\mu(B_d)\|_\infty \leq \gamma.$

We can now construct an approximating function in the same way as we did in \dnflearnprod (Cor.~\ref{cor:learn-dnf-prod}). Namely, in the third step of the algorithm we run \ptfapproxprod on $\tilde{f}(B_d)$ restricted to the variables in $M$, to obtain, with probability at least $5/6$, a bounded function $g$ such that $$\|\hat{f}_\mu(B_d) - \hat{g}_\mu(B_d)\|_\infty \leq 5\gamma = 5\eps'/(2 (2-c)^{d/2} s + 1).$$ Then, by Theorem \ref{th:dnf-fourier-approx-bound}, $$\E_\mu[|f-g|] \leq \|\hat{f}_\mu(B_d) - \hat{g}_\mu(B_d)\|_\infty \cdot (2 (2-c)^{d/2} s + 1) +4 \eps' = 5\gamma \cdot (2 (2-c)^{d/2} s + 1) + 4\eps' = 9\eps' = \eps. $$ Hence, with probability at least $1/2$, we will output $g$ such that $\pr_\mu[\sgn(g) \neq f]\leq \eps$.

To analyze the running time of our algorithm, we note that for a fixed constant $c$, $$1/\gamma = O\left((s/\eps)^{(1/\log{(2/(2-c))} +1)/2 }\right) = \poly(s/\eps).$$ The first step of the algorithm takes $\tilde{O}(n \gamma^{-4})$ time. According to Theorem \ref{th:low-degree}, the second step can be done in $n\cdot  |M|^d \cdot \poly(|M|/\gamma) = n\cdot (s \log{(s/\eps)})^{O(\log{(s/\eps)})}$ time steps (the factor $n$ comes from the fact that obtaining an individual random example and restricting it to the influential variables takes $O(n)$ time steps). According to Corollary \ref{cor:learn-dnf-prod}, the third step can be done in $n \cdot \poly(|M|,1/\gamma) =  n \cdot \poly(s/\eps)$ time steps.  Altogether, we obtain the claimed bound on the running time.
\end{proof}

\end{document}